\documentclass[letterpaper, 10 pt, conference]{ieeeconf}  

\IEEEoverridecommandlockouts                              

\usepackage{cite}
\usepackage{hyperref}
\usepackage[pdftex]{graphicx}
\usepackage{epstopdf}
\usepackage{amsmath}

\usepackage{amsthm}
\usepackage{amssymb}
\usepackage{array}
\usepackage{wasysym}
\usepackage{bm}
\usepackage{units}
\usepackage[per-mode=symbol]{siunitx}
\usepackage{url}
\usepackage[usenames, dvipsnames]{color}
\usepackage{epstopdf}
\usepackage[ruled]{algorithm2e}
\usepackage{multirow}
\usepackage{booktabs}
\usepackage{lineno}
\usepackage{wasysym}
\usepackage{epstopdf}

\newtheorem*{theorem*}{Theorem}

\title{\LARGE \bf Control of A High Performance Bipedal Robot using Viscoelastic Liquid Cooled Actuators}

\author{Junhyeok~Ahn$^{1}$, Donghyun~Kim$^{2}$, Seunghyeon~Bang$^{3}$, Nick~Paine$^4$, Luis~Sentis$^{3\dag}$
\thanks{$^{1}$ J. Ahn is with the Department of  Mechanical Engineering, The University of Texas at Austin, TX, 78712, USA {\tt\small junhyeokahn91@utexas.edu}}%
\thanks{$^{2}$ D. Kim is with the Department of  Mechanical Engineering, Massachusetts Institute of Technology, MA, 02139, USA {\tt\small alex.d.kim0821@gmail.com}}%
\thanks{$^{3}$S. Bang and L. Sentis are with the Department of Aerospace Engineering and Engineering Mechanics, The University of Texas at Austin, TX, 78712, USA {\tt\small bangsh0718@utexas.edu, lsentis@austin.utexas.edu}}%
\thanks{$^{4}$N. Paine is with Apptronik, Austin, USA {\tt\small npaine@apptronik.com}}
\thanks{$^{\dag}$ L. Sentis is the corresponding authors.}
}

\begin{document}

\maketitle
\thispagestyle{empty}
\pagestyle{empty}

\begin{abstract}

This paper describes the control, and evaluation of a new human-scaled biped robot with liquid cooled viscoelastic actuators (VLCA). Based on the lessons learned from previous work from our team on VLCA, we present a new system design embodying a Reaction Force Sensing Series Elastic Actuator and a Force Sensing Series Elastic Actuator. These designs are aimed at reducing the size and weight of the robot's actuation system while inheriting the advantages of our designs such as energy efficiency, torque density, impact resistance and position/force controllability. The robot design takes into consideration human-inspired kinematics and range-of-motion, while relying on foot placement to balance. In terms of actuator control, we perform a stability analysis on a Disturbance Observer designed for force control. We then evaluate various position control algorithms both in the time and frequency domains for our VLCA actuators. Having the low level baseline established, we first perform a controller evaluation on the legs using Operational Space Control. Finally, we move on to evaluating the full bipedal robot by accomplishing unsupported dynamic walking. 
\end{abstract}

\section{INTRODUCTION}
\label{sec:introduction}

During the DARPA Robotics Challenges (DRC), several international teams explored the use of humanoid robots in emergency response tasks. Some of the robots employed in the DRC were humanoid bipeds. Biped robots could have benefits over other embodiments in tasks such as maneuvering in tight spaces. 

In terms of design, the DRC humanoid robots SCHAFT\cite{ito2014development} and JAXON\cite{kojima2015development} were built to improve heat dissipation and thermal management capabilities through the use of liquid-cooled electric actuators. Another two humanoids, ESCHER\cite{knabe2015design} and VALKYRIE\cite{Radford:2015ca, paine2015actuator} were built with force control and ground impact resistance via the use of Series Elastic Actuators (SEAs). These robots possess actuated ankles for locomotion which is good for standing manipulation but at the same time result in a larger leg distal mass. In turn, it slows down the stepping strides and makes the robot heavier. Biped robots like ATRIAS, CASSIE, and MERCURY \cite{hubicki2016atrias, hurst2019, dh_ijrr_2019} are either void of ankle actuation or have small motors for the ankles. This choice is intended to significantly reduce leg distal mass and allow for faster swing cycles during dynamic locomotion. Faster swing cycles can have benefits for collision recovery.

\begin{figure}
    \centering
    \includegraphics[width=0.7\linewidth]{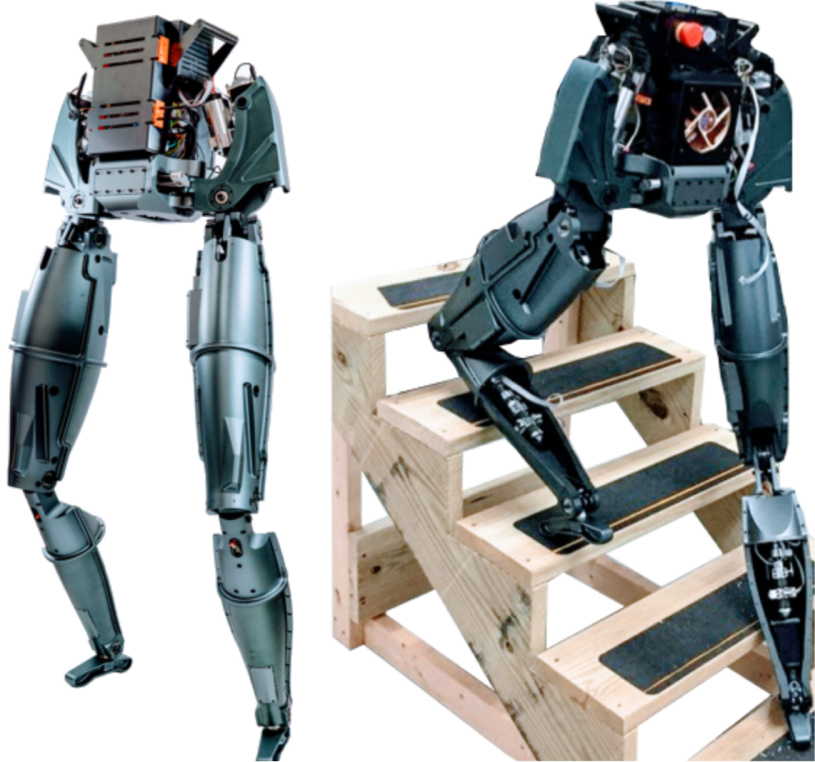}
    \caption{DRACO is a high performance biped equipped with liquid cooled VLCAs. It is designed to nimbly maneuver in cluttered environment.}
    \label{fig:Draco}
    \vspace{-4mm}
\end{figure}

\begin{figure*}
    \centering
    \includegraphics[width=\linewidth]{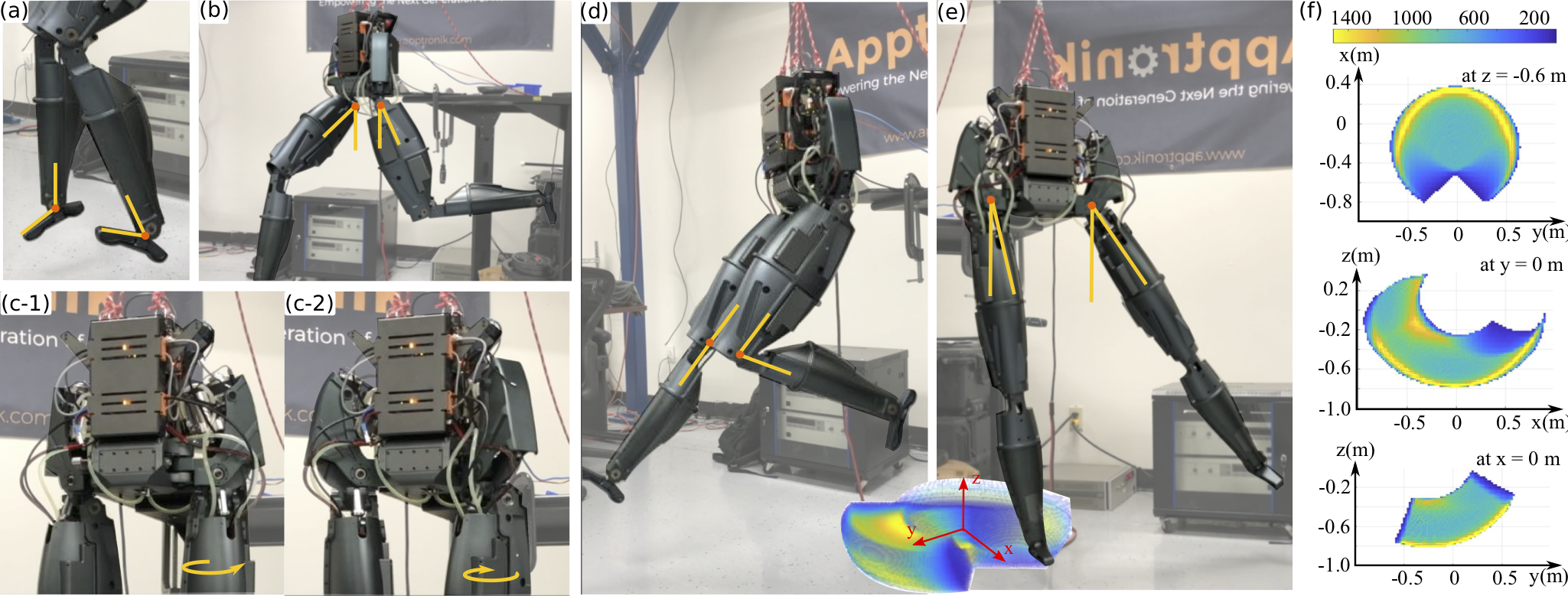}
    \caption{DRACO's joints approximate the range of motion of the adult human body with 1.5, 1.19, 2.27, 2.14, 1.55 \si{\radian} for hip rotation, abduction, flexion, knee flexion and ankle flexion respectively. (a), (b), (c-1), (c-2), (d), and (e) shows ankle flexion/extension, hip flexion/extension, hip external rotation, hip internal rotation, knee flexion/extension and hip abduction/adduction. Color maps in (e) and (f) represent the density score for the workspace of the foot.}
    \label{rom}
    \vspace{-4mm}
\end{figure*}

ATRIAS was built to mimic a spring-mass model for facilitating unsupported passive-ankle dynamic walking. Since then it has achieved walking in various terrains and running \cite{ramezani2014performance,sreenath2013embedding}. However, its controller assumes a mechanical approximation of the robot design to an idealized pendulum model with most of its mass located at the Center-of-Mass (CoM). This can limit the distribution of masses and topology of unsupported passive-ankle robots. Besides ATRIAS there is not much research published on the design of robots capable of unsupported passive-ankle walking. 

Motivated by agile locomotion, our paper introduces new bipedal robot technology with emphasis on power-dense electric actuation and the achievement of unsupported dynamic balancing. DRACO, illustrated in Fig.~\ref{fig:Draco}, is an adult-size lower-body biped robot designed to maneuver nimbly in cluttered environments. The lower-body mechanical architecture was designed to reduce volume and weight for the rated height and payload. Three electric actuators drive each leg's hip motions. One located near the pelvis providing hip rotation. Another one located on the outer lateral hip providing hip abduction/adduction. And another one located on the upper thigh providing hip flexion/extension. Knee flexion/extension is provided by an actuator located on the lower part of the thigh. Finally, ankle pitch is provided by a small actuator located on the leg's calf with the purpose of statically balancing the robot. Thus during walking, DRACO is designed to perform passive-ankle dynamic locomotion without relying on active ankle torques.

DRACO is actuated by Viscoelastic Liquid Cooled Actuators (VLCAs) which include viscoelastic elements in the drivetrain in order to improve joint position controllability as reported in \cite{kim2018investigations}. Liquid cooled Reaction Force Sensing Series Elastic Actuators (RFSEA) are used for the hip and knee joints, reducing actuation weight while increasing energy efficiency, torque density, impact resistance and position/force controllability. Liquid-cooled Force Sensing Series Elastic Actuators (FSEA) drive the pitch ankle joints for ankle flexion/extension control. To increase heat dissipation on the electric motors, these VLCAs use liquid-cooling motor jackets \cite{paine2015design} enabling 2.5x higher continuous torques on all joints compared to conventional electric actuators.

For actuator control, we use a decoupled control strategy as described in \cite{paine2015actuator}, which relies on a rigid joint model. Each low level actuator controller acts as an idealized force or joint position source to facilitate the use of control architectures for multi Degree Of Freedoms (DOFs), such as Whole Body Control (WBC)\cite{sentis2005synthesis}. Previously, we showed high fidelity control of SEAs via Disturbance Observer (DOB) controllers designed with the assumption of a time invariant nominal SEA model \cite{paine2014design} and the improvement of joint position controllability by placing viscoelastic materials on the actuator's drivetrain \cite{kim2018investigations}. In this paper, we prove the robust stabilization capabilities of our DOB controllers. We evaluate various joint position feedback controllers depending upon 1) using either motor and spring encoders versus linear potentiometers for feedback, and 2) whether to include force inner feedback loop to decrease mechanical friction and stiction. Finally, we implement and test Operational Space Control (OSC) \cite{khatib1987unified} and dynamic balance control using WBC \cite{dh_ijrr_2019}. We demonstrate accurate and stable actuator position tracking in the operational space and unsupported dynamic balancing with well-regulated motor temperatures thanks to the liquid cooling system.

The main contribution of this paper is on the control and evaluation of a high performance biped robot with VLCAs that can achieve unsupported passive-ankle balancing. Another contribution is the study of stability and performance analysis for the proposed joint and multi-DOF controllers. The remainder of the paper is organized as follows: Section~\ref{sec:mechatronic_design} presents the mechatronic design of DRACO. Stability and performance of VLCA are studied in Section~\ref{sec:single_actuator_controllers}. OSC and unsupported dynamic balancing via WBC are described in Section~\ref{sec:multi_dof_controllers}. Finally, Section~\ref{sec:conclusion} concludes the paper.


\section{Mechatronic Design}
\label{sec:mechatronic_design}
\begin{figure*}[t]
    \centering
    \includegraphics[width=1.0\linewidth]{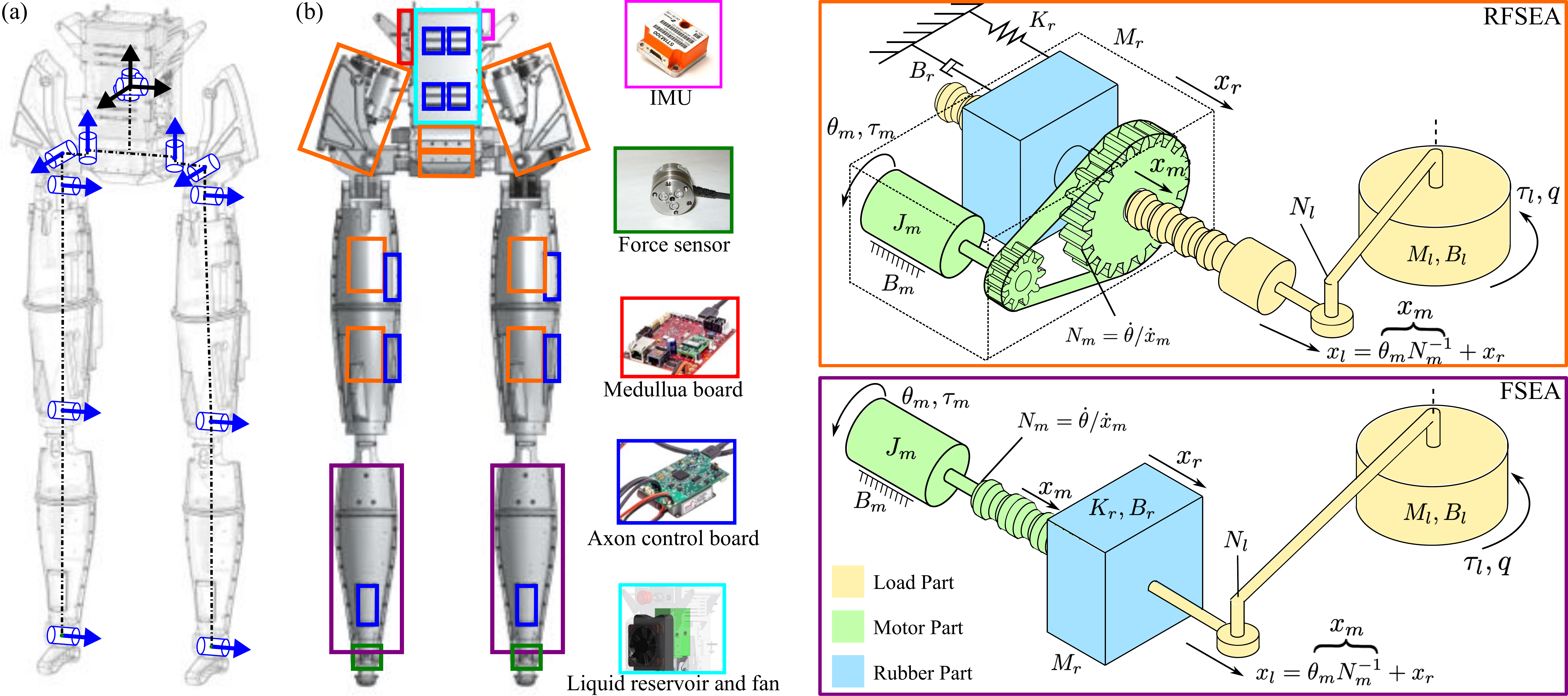}
    \caption{(a) shows generalized coordinate of the robot. Floating base frame and actuated DoFs are represented as black and blue axis, respectively. (b) Mechatronic Parts and Schematic diagrams of the VLCAs.}
    \label{fig:mechanical_overview}
    \vspace{-4mm}
\end{figure*}

\subsection{Mechatronics Overview}
\label{sec:mechatronics_overview}

The DRACO biped is 1.30 \si{\meter} tall, weighs 37 \si{\kg}, and achieves a similar range of motion than an adult human regarding leg and hip motions. DRACO has ten actuated DOFs including three for the hip structure, one for the knee and one for the ankle. The range of motion of each DOF and workspace of the robot's foot are shown in Fig.~\ref{rom}. Generalized coordinate of DRACO, and mechatronic components including actuators, auxiliary sensors and electronic boards are illustrated in Fig.~\ref{fig:mechanical_overview}.

Unlike many humanoid robots, it does not have ankle roll actuation. This allows to significantly reduce distal mass and therefore enhance swing speed motion. Based on the lessons learned from previous works \cite{slovich2012building, Jung, paine2014design}, the hardware has been designed with performance and mechanical safety consideration while reducing it's overall weight and the risk of overheating motors. 

\subsection{Viscoelastic Liquid Cooled Actuators}%
\label{sub:viscoelastic_liquid_cooled_actuator}

To achieve the design objectives and protection against external impacts, VLCAs are employed to actuate the robot joints. VLCA is a family of prismatic SEA with a viscoelastic material instead of metal springs and active liquid cooling, first introduced in our previous work \cite{kim2018investigations}. There we investigated its power density, energy efficiency, high-fidelity force control and joint position control. In this section, we introduce RFSEA type VLCA actuators and FSEA type VLCA actuators used in DRACO.

The schematic diagrams of the VLCAs and the nomenclatures are shown in  Fig.~\ref{fig:mechanical_overview}. In the diagram, $J_m(\si{\kilo\gram\square\meter})$, $B_m(\si{\newton\second\per\meter})$, $M_r(\si{\kilo\gram})$, $B_{r}(\si{\newton\second\per\meter})$, $M_l(\si{\kilo\gram})$ and
$B_l(\si{\newton\second\per\meter})$ are motor inertia, motor damping coefficient, elastomer mass, elastomer damping coefficient, load mass and load damping coefficient. $\theta_m(\si{\radian})$, $x_r(\si{\meter})$ and $x_l(\si{\meter})$ are displacement of the motor, elastomer and the output load which are the actuator states. $\tau_m (\si{\newton \meter})$ is motor torque which is the actuator input, and $\tau_l(\si{\newton \meter})$ and $q(\si{\radian})$ are the joint torque and position which are the joint output. $N_m (\si{\radian \per \meter})$ ($N_l(\si{\radian \per \meter})$, respectively) is the speed reduction ratio of the motor (output joint, respectively) provided by the ball screw (the actuator position, respectively). To measure the actuator states, we place a quadrature encoder at motor and elastomer side to measure motor angle, $\theta_{m}$, and elastomer deflection, $x_r$. In addition, we adopt a linear potentiometer to measure absolute position, $x_l$, of the actuator.

The RFSEA transmits mechanical power when the BLDC motor turns a ball nut via a low-loss timing belt and pulley, which causes a ball screw to exert a force to the actuator’s output. Rigid assembly consists of the electric motor, the ball screw, and the ball nut connected in series to the compliant viscoelastic element, which in turn connects to the mechanical ground of the actuator. When the actuator exerts a force, it causes the viscoelastic element to contract and extend. The liquid cooling system allows to increase the maximum continuous current by a factor of 2.5 without thermal failure.

The FSEA liquid cooled actuator transmits mechanical via a BLDC motor in series to the ball screw. The FSEA type VLCA includes a compliant element between the ball screw and the actuator output. As a result of the drivetrain, it provides a long, thin and lightweight design that is ideal to incorporate in the calf of the leg. For more detailed information, readers are referred to our previous work \cite{kim2018investigations}.

\section{Actuator Control}
\label{sec:single_actuator_controllers}

In Fig.~\ref{fig:block_diagram}(a), we outline our overall robot control structure which contains a multi-joint control block coordinating multiple decoupled joint controllers. In order to control the robot's dynamic locomotion behavior effectively, the low level actuator controllers are designed to deliver certain high performance specifications. In our recent work, \cite{kim2018investigations}, we studied joint position controllability and torque controllability for liquid cooled viscoelastic actuators. In this section, we extend this analysis by studying the effects of different types of sensors for joint position feedback as well as performing a stability analysis of DOBs used for force control.
\begin{figure}
    \centering
    \includegraphics[width=\linewidth]{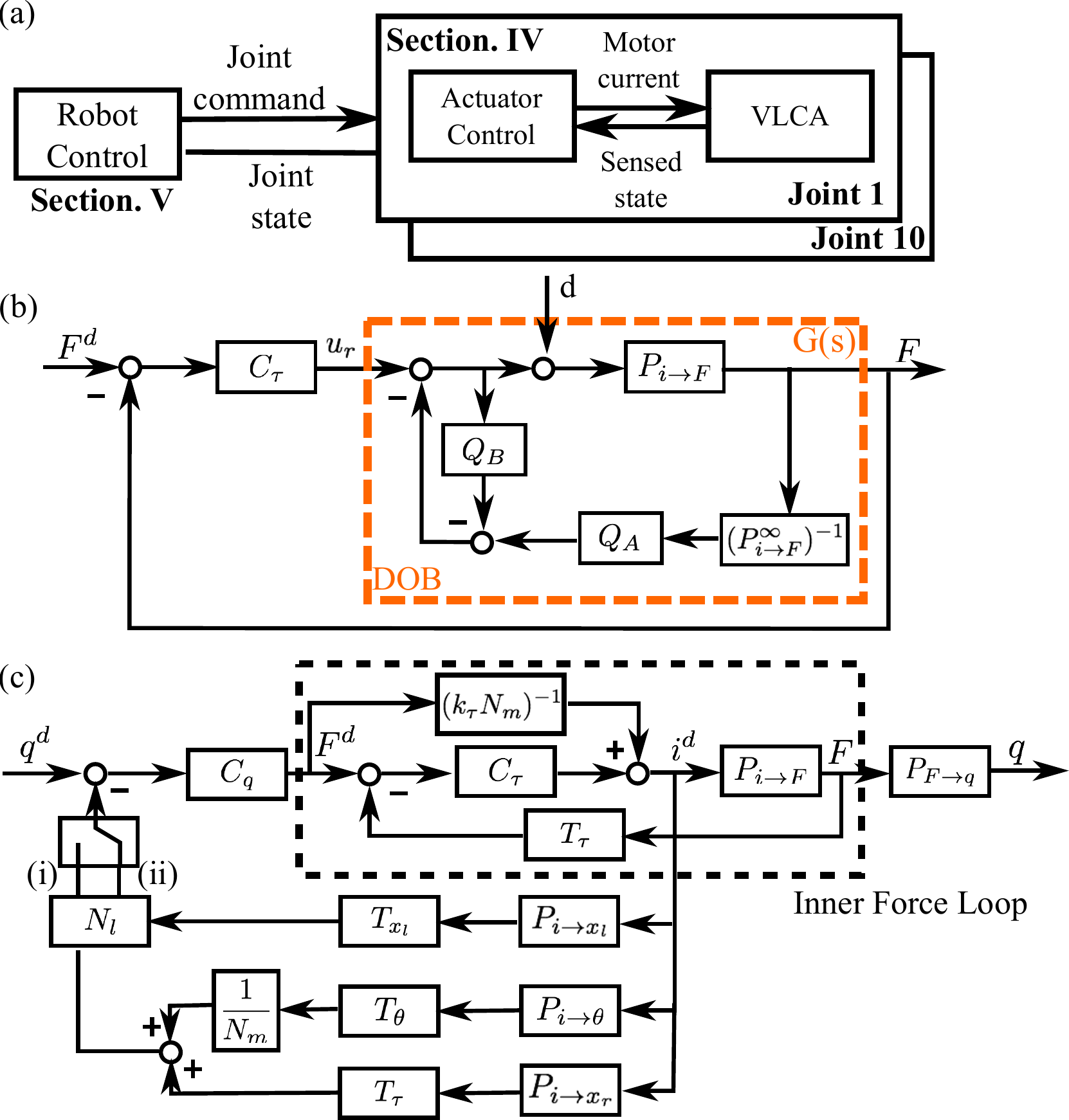}
    \caption{(a) shows our decoupled control approach. (b) represents a block diagram for force feedback control incorporating a DOB structure. (c) shows a block diagram of our actuator position feedback controller.}
    \label{fig:block_diagram}
    \vspace{-5mm}
\end{figure}
\subsection{Model \& Identification}
\label{sub:system_modeling_identification}

Let us consider the transfer functions for RFSEA actuators \cite{park2016dynamic} using the nomenclatures shown in Fig.~\ref{fig:mechanical_overview},
\begin{equation}
    \label{eq:plants1}
    \begin{split}
        P_{i \rightarrow \theta_m}(s) &= \frac{k_{\tau} P_m(s)\Big(P_r(s) + P_l(s)\Big)}{N_m^2P_m(s) + P_l(s) + P_r(s)} \\
        P_{i \rightarrow x_r}(s) &= -\frac{k_{\tau} P_m(s) P_r(s)}{N_m \Big( N_m^2P_m(s) + P_l(s) + P_r(s) \Big)} \\
        P_{i \rightarrow F}(s) &= K_r P_{i \rightarrow x_r}(s) \\
        P_{F \rightarrow q}(s) &= N_l P_l,
    \end{split}
\end{equation}
where $i(\si{\ampere})$, $F(\si{\newton})$ and $k_{\tau}(\si{\newton\meter\per\ampere})$, are motor current, actuator force and motor constant. $P_{\circ \rightarrow \triangle}(s)$ represents transfer functions with input signal $\circ$ and output signal $\triangle$. In addition $P_m$, $P_r$ and $P_l$ correspond to the motor, elastomer, and load transfer functions with expressions,
\begin{equation}
    \label{eq:plants2}
    \begin{split}
        P_m(s) &= \frac{1}{J_m s^2 + B_m s} \\
        P_r(s) &= \frac{1}{M_r s^2 + B_{r} s + K_r} \\
        P_l(s) &= \frac{1}{M_l s^2 + B_l s},
    \end{split}
\end{equation}
Note that $M_l$ in the above equation is indefinite since it varies with the joint configuration and the contact state of the robot reflected onto the actuator. $P_{i \rightarrow F}^{\infty}$ in Fig.~\ref{fig:block_diagram}(c) represents the same plant as $P_{i \rightarrow F}$ but with an infinite load mass $M_l \simeq \infty$, i.e. an ideal rigid contact.

$k_\tau$ and $N_m$ are obtained from product sheets, $N_l$ is derived from a pre-computed look-up table, and $K_r$ is approximated by measuring elastomer displacements given known applied forces. To obtain other parameters, we do so via system identification techniques. We generate a motor current following an exponential chirp signal, with frequencies between 0.01 $\si{\Hz}$ to 200 $\si{\Hz}$ and measure actuator force as an output signal while 1) constraining the actuator output to a fixed position ($M_l \simeq \infty$) and 2) letting the actuator to move freely. Note that by fixing the actuator output, its open loop transfer function becomes second order such that motor and elastomer parameters can be identified independently. Combining the system identification tests with constrained and free moving outputs, we identify the rest of the parameters as shown in Table~\ref{tb:actuator_id} and compute the bode plots shown in Fig.~\ref{fig:single_joint_control}(a).

\begin{table} \centering
\caption{Actuator Parameters}
\label{tb:actuator_id}
\begin{tabular}{>{\centering}m{0.09\columnwidth} %
                >{\centering}m{0.09\columnwidth} %
                >{\centering}m{0.15\columnwidth} %
                >{\centering}m{0.09\columnwidth} %
                >{\centering}m{0.07\columnwidth} %
                >{\centering}m{0.09\columnwidth} %
                >{\centering}m{0.09\columnwidth} %
                @{}m{0pt}@{}}
\specialrule{1.5pt}{1pt}{1pt}
{$J_m$}
& {$B_m$}
& {$M_l$}
& {$B_l$}
& {$M_r$}
& {$B_r$}
& {$K_r$}
&\\[1mm]
\hline
\hline 
\num[output-exponent-marker = \text{e}]{1.6e-5} &
\num[output-exponent-marker = \text{e}]{2.0e-4} &
2953 $\sim$ $\infty$ &
\num[output-exponent-marker = \text{e}]{2.0e4} &
1.3 &
\num[output-exponent-marker = \text{e}]{1.6e4} &
\num[output-exponent-marker = \text{e}]{9.5e6} & \\[1mm]
\hline
\end{tabular}
\vspace{-4mm}
\end{table}

\subsection{Force Feedback Control}
\label{sec:actuator_force_control}

Many different methods have been proposed for controlling series elastic actuators using force feedback. \cite{pratt2002series} studied high fidelity force control of SEAs measuring force via compression of a compliant element and \cite{kong2012compact, paine2015actuator} studied PID, model-based and DOB structures to achieve high fidelity force tracking. Since our actuator model considers a variable load, unknown a priori, we first design a nominal plant based on an infinite load mass assumption. Fig.~\ref{fig:block_diagram}(b) depicts our force feedback controller with a DOB where $Q_A, Q_B$ and $C_{\tau}$ correspond to low pass filters and a PD controller respectively, with expressions, 
\begin{equation}
    \label{eq:frc_ctrl}
    \begin{split}
        C_{\bullet} &= P_{\bullet} + D_{\bullet} s \\
        Q_{\bullet} &= \frac{1}{(s/\omega_c)^2 + \sqrt{2}(s/\omega_c) + 1}.
    \end{split}
\end{equation}
Here, $P_{\bullet}$ and $D_{\bullet}$ are proportional and derivative gains, and $\omega_c$ is the cut-off frequency of the filter defined by $Q_{\bullet}$. In addition, $P_{i \rightarrow F}$ represents the actual actuator plant with motor current as input and actuator force as output. $P_{i \rightarrow F}^{\infty}$ represents a model of the actuator plant using the infinite load mass assumption.

In this section, we provide a formal analysis on the robustness and stability of DOB-based controllers under uncertain loads. In order to study the performance of our DOB controller given the time varying output load we will apply perturbation theory analysis \cite{shim2016yet}. We derive the state-space equations of our DOB, $G(s)$ in Fig.~\ref{fig:block_diagram}(b) by using the method explained in \cite{shim2007state}, resulting in the equations,
\begin{equation}
    \label{eq:ss_realization_1}
    \begin{split}
        x_1 &= F\\
        \dot{x}_{1} &= x_{2} \\
        \dot{x}_{2} &= \bm{\psi}^{\top} z + \bm{\phi}^{\top} \bm{x} + \bm{g} \bm{C} (\bm{\zeta} - \bm{\xi}) + \bm{g} (u_r + d) \\
        \dot{z} &= \bm{S} z + \bm{G} \bm{x} \\
        \dot{\overline{z}} &= \overline{\bm{S}} \overline{z} + \overline{\bm{G}} \bm{x},
    \end{split}
\end{equation}
and

\begin{equation} \label{eq:ss_realization_2}
    \begin{split}
        &\tau \dot{\bm{\xi}} = \bm{A} \bm{\xi} + \frac{\bm{g}}{\overline{\bm{g}}}\bm{BC}(\bm{\zeta} - \bm{\xi}) 
        + \frac{1}{\overline{\bm{g}}}\bm{B}\big( -\overline{\bm{\psi}}^{\top} \overline{z} + \bm{\psi}^{\top} z \\
        & \qquad \:\:+(\bm{\phi} - \overline{\bm{\phi}})^{\top} \bm{x} \big)+ \frac{\bm{g}}{\overline{\bm{g}}}\bm{B}(u_r + d) \\
        &\tau \dot{\bm{\zeta}} = \bm{A} \bm{\zeta} + \bm{BC}(\bm{\zeta} - \bm{\xi}) + \bm{B}u_r.
    \end{split}
\end{equation}
Here, $u_r$ and $F$ are the input reference and the measured force respectively. $\bm{x} = [x_1, x_2]^{\top}$ and $z \in \mathbb{R}$ are the internal states and zero dynamics state of the plant $P_{i \rightarrow F}$. The 2D vectors, $\bm{\xi}$ and $\bm{\zeta}$ represent states corresponding to the filters $Q_A$ and $Q_B$ respectively. The matrices $\bm{\psi}$, $\bm{\phi}$, $\bm{g}$, $\bm{S}$, $\bm{G}$, $\bm{A}$, $\bm{B}$ and $\bm{C}$ are unknown but bounded plant parameters. In addition, the state and the plant parameters expressed with symbol overlines represent the same vector and operators for the nominal plant $P_{i \rightarrow F}^{\infty}$. For example, $\overline{\bm{S}},~\overline{\bm{G}},~\overline{\bm{\psi}},~\overline{\bm{\phi}},~\overline{\bm{g}}$ and $\overline{z}$ represent plant parameters and zero dynamics for the nominal plant. Based on the identified parameters of our system shown in table~\ref{tb:actuator_id} and the second order Butterworth filter of Eq.~\eqref{eq:frc_ctrl} with $60 \si{\Hz}$ cutoff frequency for both $Q_A$ and $Q_B$, the state-space parameters become
\begin{equation*}
\begin{split}
    &\bm{A} = \begin{bmatrix} 0&1\\-1&-\sqrt{2} \end{bmatrix},~\bm{B}= \begin{bmatrix} 0 \\ 1 \end{bmatrix}, ~\bm{C} = \begin{bmatrix} 1 & 0 \end{bmatrix}, \\
    &\overline{\bm{\psi}} = -9.53e3,~ \overline{\bm{\phi}} = 1e4[-1.49,~-0.01], \\ &\overline{\bm{S}}=-0.0433,~\overline{\bm{G}}= \begin{bmatrix}
    1 & 0
\end{bmatrix},~\overline{\bm{g}}=1.31e4.
\end{split}
\end{equation*}

Eq.~\eqref{eq:ss_realization_1} and \eqref{eq:ss_realization_2} are in the standard form for singular perturbation analysis where $\tau$ represents the perturbation parameter \cite{khalil2002nonlinear}. The variables $\bf{x},~z,~\overline{z}$ are called slow dynamics while the variables $\bm{\xi}$ and $\bm{\zeta}$ are called fast dynamics and the following theorem holds:
\begin{theorem*}
    The proposed DOB structure (Fig.~\ref{fig:block_diagram}(b)) is robustly stable and converges to the performance of the nominal plant under uncertain loads.
\end{theorem*}
\begin{proof} \cite{shim2007state} proved that if the unknown variables $\bm{\phi}$, $\bm{\psi}$,$\bm{S}$, $\bm{G}$ and $\bm{g}$ are bounded with $\bm{g} \neq  0$, there exists a $\tau^* > 0 $ such that, for all $0 < \tau < \tau^*$, the DOB structure is robustly stable if 1) the zero dynamics of the actual actuator plant are stable, 2) the boundary-layer subsystem, Eq.~\eqref{eq:ss_realization_1}, is exponentially stable.
In our case,
\begin{enumerate}
    \item The actual plant, $P_{i \rightarrow F}$, has stable zeros given the identified actuator parameters of Table~\ref{tb:actuator_id}, 
    \item The system matrix of Eq.~\eqref{eq:ss_realization_2} is 
    \begin{equation*}
        \begin{bmatrix}
            \bm{A} - \frac{\bm{g}}{\overline{\bm{g}}} \bm{B C} & \frac{\bm{g}}{\overline{\bm{g}}} \bm{B C} \\
                -\bm{B C} & \bm{A + B C}
        \end{bmatrix},
    \end{equation*} and its characteristic polynomial is $p_a(s)p_b(s)$ with
    \begin{equation*}
        \begin{split}
            p_a(s) &= \begin{vmatrix} s \bf{I} - \bf{A} \end{vmatrix} \\
            p_b(s) &= \begin{vmatrix} s \bf{I} - \bf{A} + \frac{\bm{g}-\overline{\bm{g}}}{\overline{\bm{g}}} \bm{BC} \end{vmatrix},
        \end{split}
    \end{equation*}
    where $\bm{g}=\frac{3.2e4 M_l}{2.4 M_l + 3.1}$. This system matrix is always Hurwitz, since $p_a(s)$ and $p_b(s)$ have negative roots for all possible $M_l$, which results in the exponential stability for Eq.~\eqref{eq:ss_realization_2}.
\end{enumerate}
\vspace{-4mm}
\end{proof}

\subsection{Position Feedback Control}%
\label{sub:joint_position_control}

\begin{figure}
    \centering
    \includegraphics[width=1.0\linewidth]{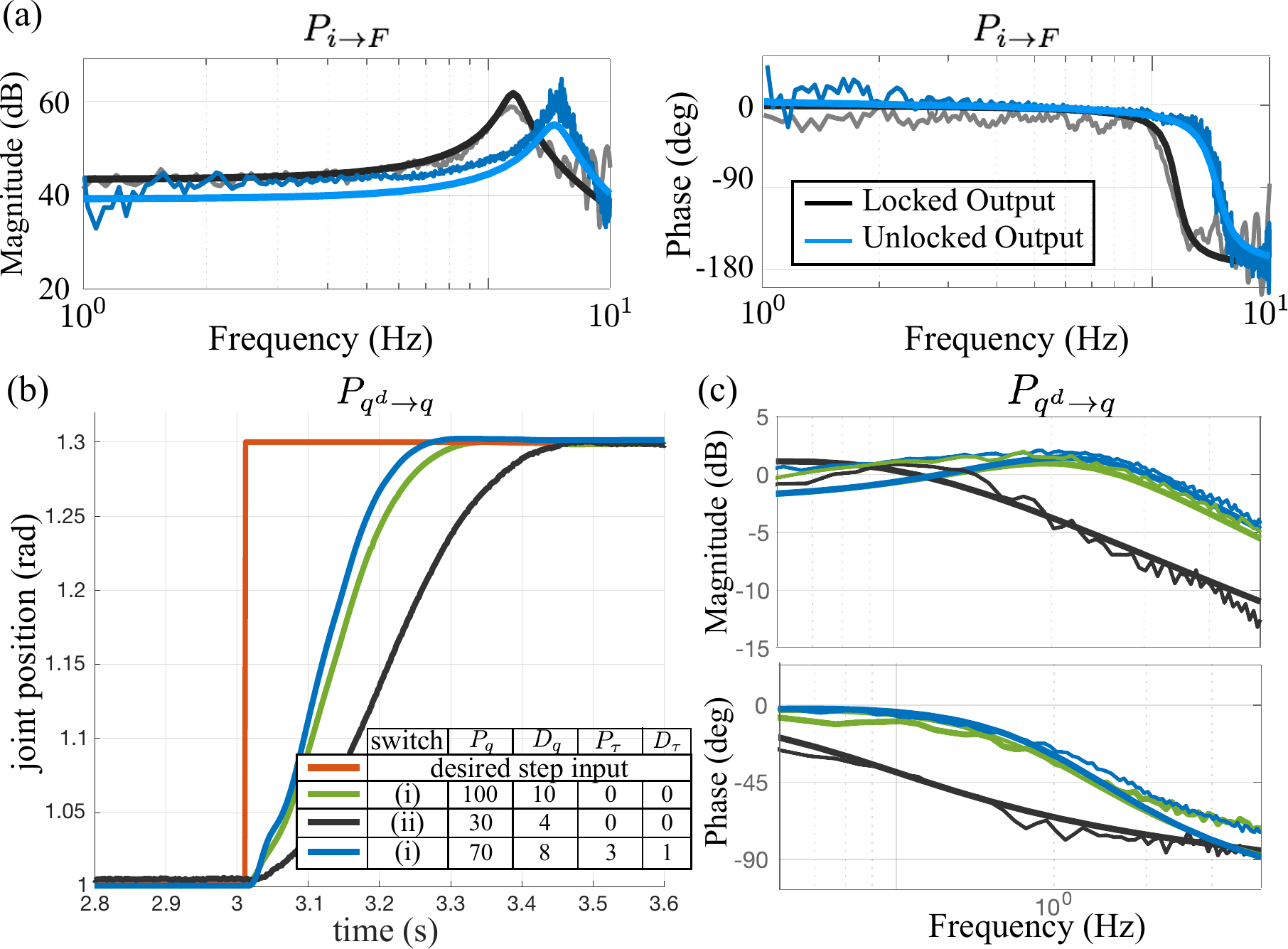}
    \caption{ (a) shows experimental and estimated bode plots of input current versus output force based on two different actuator output set-ups for system identification. (b) shows experimental step responses of three different joint position feedback controllers. Selected gains and the switch to measure feedback signals are shown in the table. (c) shows experimental and estimated joint position control performance for each controllers.}
    \label{fig:single_joint_control}
    \vspace{-4mm}
\end{figure}
\begin{figure*}
    \centering    \includegraphics[width=\linewidth]{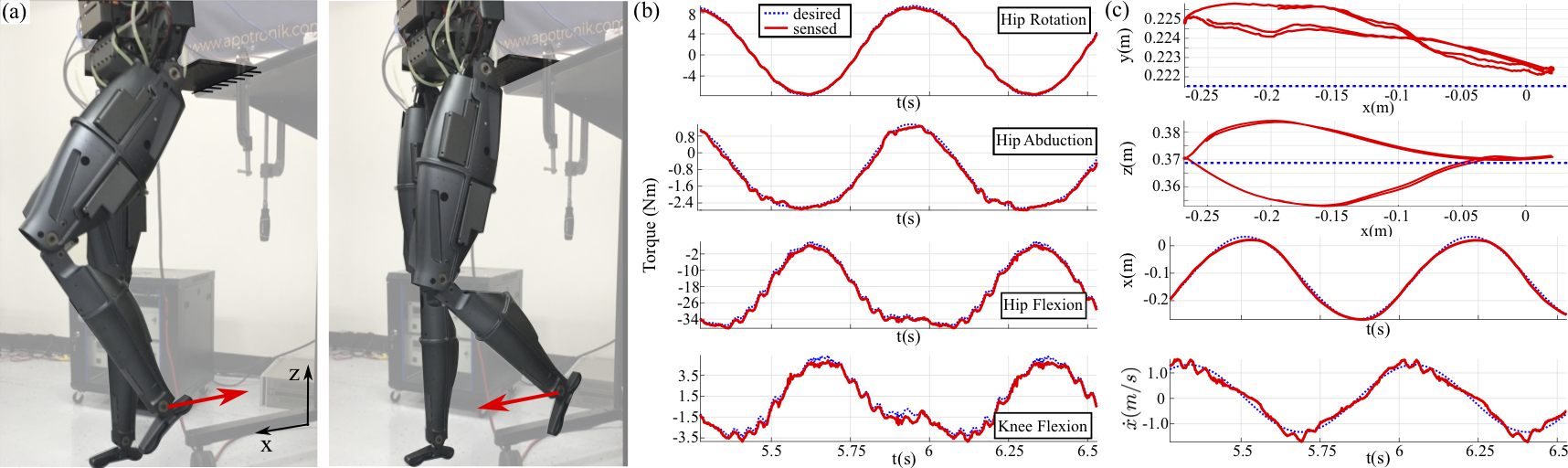}
    \caption{ \textbf{Operational Space Control Test.} (a) shows DRACO moving its left ankle following a Cartesian trajectory. (b) shows desired and measured joint torques for tracking performance. (c) shows desired and measured Cartesian position tracking performance. Both torque and Cartesian position trackers are tracking closely the reference trajectories.}
    \label{fig:osc}
    \vspace{-4mm}
\end{figure*}

In this subsection, we design different types of position controllers: 1) by measuring actuator position by either adding up motor quadrature encoders and elastomer quadrature encoders, or directly using a linear potentiometer, and 2) by including force feedback control within the position control loop. Fig.~\ref{fig:block_diagram}(c) shows our joint position control structure using PD control, $C_\bullet$, and including time delays, $T_{\bullet} = e^{-T_{\bullet}s}$. The switch labeled (i) uses the option with motor and elastomer quadrature encoders and (ii) uses the option with a linear potentiometer to measure actuator position. The force feedback control loop enclosed with a black dotted box where $C_\tau$ can be set to zero if we want to remove this loop from the joint position controller. The transfer functions of the close loop systems for each switch option can be derived from inspecting the block diagram resulting in:
\begin{align}
        \label{eq:non_collocated}
        P_{q \rightarrow q^d}^{(i)} &= \frac{C_q P_{F \rightarrow F^d} P_{F \rightarrow q}}{1 + D_q C_q P_{F \rightarrow F^d} P_{F \rightarrow q}} \\
        \label{eq:collocated}
        P_{q \rightarrow q^d}^{(ii)} &= \frac{P_{F \rightarrow q} P_{F \rightarrow F^d} C_q}{1 + D_q C_q P_{F \rightarrow F^d} N_l P_{i \rightarrow \theta} / (N_m P_{i \rightarrow F})},
\end{align}
where $P_{F \rightarrow F^d}$ represents the transfer function of the inner force control loop which has the expression,
\begin{equation}
    P_{F \rightarrow F^d} = \frac{P_{i \rightarrow F} (k_\tau N_m + C_\tau)}{1+P_{i \rightarrow F} D_\tau C_\tau}.
\end{equation}

Based on the block diagram, we design three different joint position controllers to compare their performance: 1) using motor and elastomer quadrature encoder feedback without inner force feedback control ($P_{q \rightarrow q^d}^{(i)},~C_\tau = 0$), 2) using linear potentiometer feedback without inner force feedback control ($P_{q \rightarrow q^d}^{(ii)},~C_\tau=0$), and 3) using motor and elastomer quadrature encoder feedback with inner force feedback control ($P_{q \rightarrow q^d}^{(i)},~C_\tau \neq 0$).

To compare the closed loop systems, we empirically choose gains such that the position control loops of controllers 1) and 2) from above behave as critically damped systems. To increase feedback gains, we first increase $D_q$ before the system gets unstable and then choose the highest stable $P_q$, such that the step response of the joint position controller, $P_{q^d \rightarrow q}$, does not overshoot.

In the case of using position control with the inner force control loop, i.e. controller 3) from previous, gain selection become more complex due to the dependencies between $C_\tau$ and $C_q$. We observed that the transfer function of a joint position control structure embedding a force control loop could be represented as the multiplication of two second order systems \cite{zhao2018impedance}. We then proposed a method to make the combined system critically damped given a desired natural frequency. For our comparative analysis, we increase the force loop gains and decrease the joint position control gains according to stability constraints. In this way, we emphasize the role of the inner force control loop to see its effect. 

We now compare controllers 1), 2) and 3), as represented with green, black and blue color, respectively, in Fig.~\ref{fig:single_joint_control}. In the figure, we choose gains to make the closed loop systems critically damped and match the natural frequency of controller 3) to controller 1).
Controller 1) performs better than 2) since the quadrature encoders give higher quality signals than the linear potentiometer. We now analyze controller 3) based on inner force feedback control. We notice that controller 3) only allows for smaller values of $P_{q}$ and $D_{q}$ than controller 1) for stability reasons. In conclusion, although the use of inner force feedback control reduces friction effects, it ends up reducing joint position gains which decrease position accuracy. As a result, for our locomotion tests we use controller 1) instead of controller 3).


\section{Robot Control}%
\label{sec:multi_dof_controllers}
Building on our actuator control study above, we devise and test two multi-joint controllers for DRACO. We will first evaluate DRACO's performance using and instance of Operational Space Control \cite{khatib1987unified}. After that we will evaluate DRACO using our newest unsupported dynamic locomotion controller, which consists of two parts, WBC and Time-To-Reversal (TVR) planner \cite{dh_ijrr_2019}.

\subsection{Operational Space Control}%
\label{sub:operational_space_controller}

For this test, we first fix DRACO to a table as shown in Fig.~\ref{fig:osc}(a) and generate a Cartesian trajectory for the robot's left ankle to follow in the forward direction. We command a sinusoidal trajectory with amplitude of 0.3 \si{\meter} and frequency of 1.4 \si{\Hz}. The lateral and vertical Cartesian directions of the ankle are controlled to stay at a fix point. Torque commands for each robot joint is computed according to the OSC control law,
\begin{equation}
    \bm{\tau} = \bm{A}\overline{\bm{J}}(\ddot{\bm{x}} + K_p \bm{e} + K_d \dot{\bm{e}} - \dot{\bm{J}}\dot{\bm{q}}) + \bm{b} + \bm{g},
\end{equation}
where $\bm{A}$, $\bm{b}$, $\bm{g}$, $\dot{\bm{q}}$ and $\bm{\tau}$ are inertia, coriolis, gravity
forces, joint velocities and joint torque commands written with respect to the robot's generalized coordinates. $\ddot{\bm{x}}$, $\bm{e}$ and $\dot{\bm{e}}$ are desired Cartesian trajectory accelerations, and Cartesian position and velocity errors. $\bm{J}$ is for the jacobian of the left ankle and $\overline{\bm{J}}$ is for the dynamically consistent pseudo inverse, which is defined as $\overline{\bm{J}} \triangleq \bm{A}^{-1}\bm{J}^{\top}(\bm{JA}^{-1}\bm{J}^{\top})^{-1}$.
The robot then sends the computed joint torques command through the EtherCAT network using the embedded Axon boards for joint control. The Axon boards implement each a torque controller as described in Section \ref{sec:actuator_force_control}. The result is shown in Fig.~\ref{fig:osc} and demonstrates the performance of OSC on DRACO. Because of the use of a DOB, the joint controllers display robustness despite load uncertainty.

\subsection{Unsupported Dynamic Balancing}%
\label{sub:dynamic_balancing_controller}

\begin{figure*}
    \centering
    \includegraphics[width=1.0\linewidth]{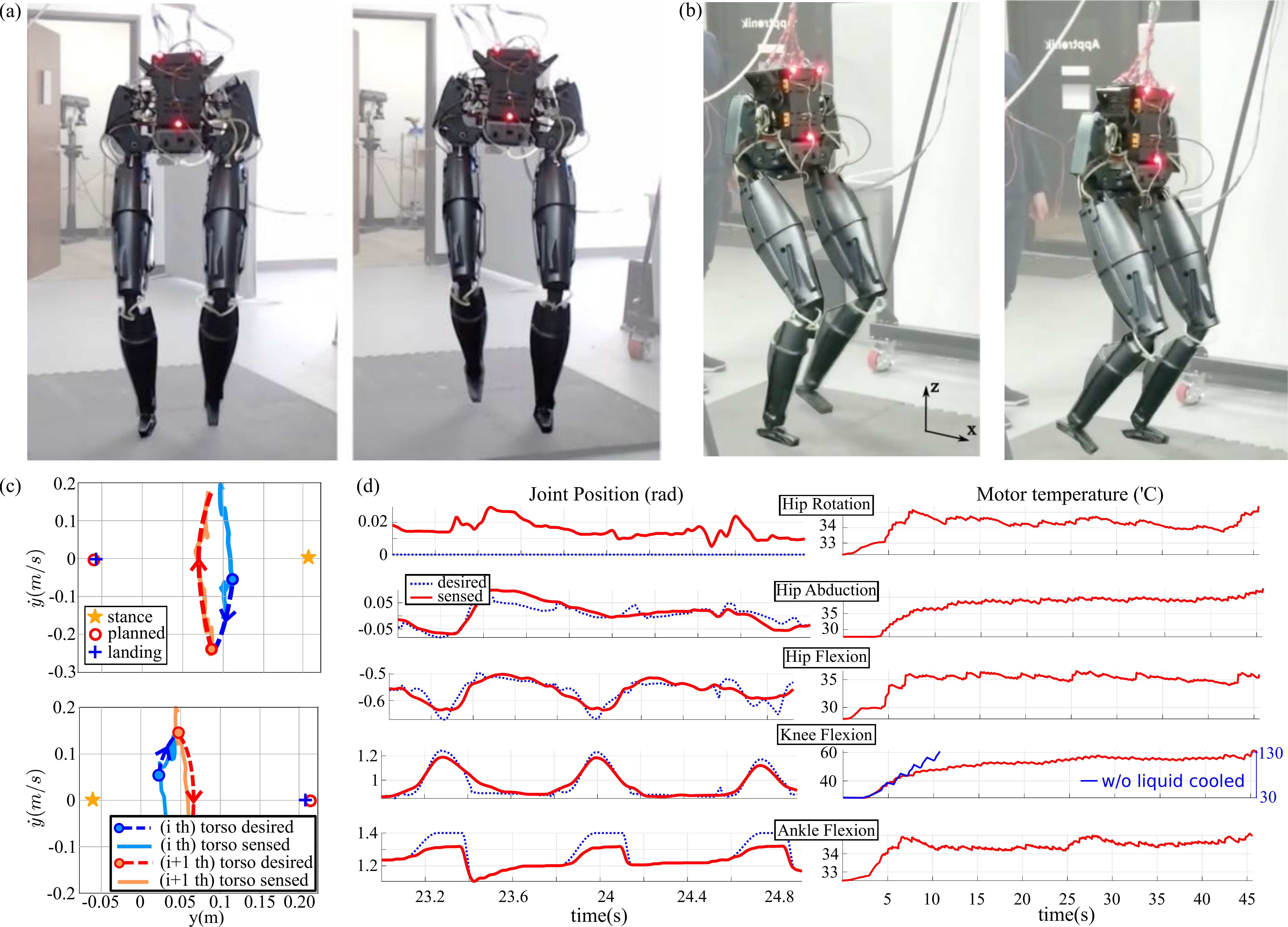}
    \caption{The sequences (a) and (b) show DRACO dynamically self-balancing without any mechanical support. (c) shows the phase space plot of the robot's CoM in the lateral direction with respect to the robot's frame for two consecutive balancing steps. (d) shows desired and measured joint positions as well as estimated core winding motor temperatures during dynamic balancing.}
    \label{fig:dynamic_balancing}
    \vspace{-4mm}
\end{figure*}

Here, we demonstrate the ability of DRACO to achieve unsupported dynamic balancing by means of the WBC and TVR algorithms (see Fig. \ref{fig:walking_block}). Dynamic balancing is achieved via stabilizing leg contact changes (coordinated by the state machine block) triggered by either pre-defined temporal specifications or foot contact sensors. The sequence of contact phases is represented by a list of tuples specifying the phase name and its time duration, i.e. \{($DS,~0.01 \si{\second}$) $\rightarrow$ ($LF_i,~0.03 \si{\second}$) $\rightarrow$ ($SW_i,~0.33 \si{\second}$) $\rightarrow$ ($LN_i,~0.01 \si{\second}$)\}\footnote{$DS,~LF_i,~SW_i$ and $LN_i$ mean Double Support, Lifting, Swing, and Landing phases respectively. The subscript $i$ represents the swing leg type, either the robot's right leg or its left leg.}. In addition, DRACO can detect sudden velocity changes on its ankle movement as a trigger mechanism to detect contact. We use ankle velocity trigger as sensors to terminate the  $SW_i$ phase.

The $\mathbf{x}_k$ symbols represent a set of operational space tasks. Arrows, $\bm{x}_k \rightarrow \bm{x}_l$, express priorities in that the left set of tasks $x_k$ has higher priority than the right set of tasks $x_l$ for arbitrary $k$ and $l$ indexes. In turn, WBC handles priorities by solving a prioritized inverse kinematics problem. Below are the task and priority assignments to the phases that we use for the robot's dynamic balancing behavior:  
\begin{align*}
    DS,~ LN,~ LF &: \{ \bm{x}_1 \rightarrow \bm{x}_3 \} \\
    SW &: \{ \bm{x}_1 \rightarrow \bm{x}_2 \rightarrow \bm{x}_3 \rightarrow \bm{x}_4 \},
\end{align*}
where the task sets are defined as
\begin{itemize}
    \item $\bm{x}_1 \in \mathbb{R}^{2}$: Right and Left Hip Rotation Task
    \item $\bm{x}_2 \in \mathbb{R}$: Swing Foot Ankle Flexion Task 
    \item $\bm{x}_3 \in \mathbb{R}^{3}$: Torso Roll, Torso Pitch, Torso Height Task
    \item $\bm{x}_4 \in \mathbb{R}^{3}$: Swing Foot Position Task
\end{itemize}
The Right and Left Hip Rotation Task and the Torso Roll, Pitch, and Height Task are set to $[0 \si{\degree},~0 \si{\degree}]$ and $[0 \si{\degree}, 0 \si{\degree}, 1 \si{\meter}]$, respectively, in order to make DRACO face forward and maintain its torso upright. For the $SW$ task assignment, we incorporate the Swing Foot Ankle Flexion Task with a desired value of $80 \si{\degree}$ in order to detect sudden velocity changes when touching the ground. The swing Foot Position Task is driven by b-spline trajectory computation that steers the swinging foot to a desired landing location given by the TVR planner. After all operational space tasks are specified, the WBC controller shown in Fig.~\ref{fig:walking_block} provides the computation of sensor-based feedback control loops and motor commands to achieve the desired goals. As a result, the entire body of the robot, the actual plant, will execute the commands to dynamically balance without support. For this particular experiment we use both an IMU in combination with a motion capture system for CoM state estimation.

The computed motor commands are then sent out to the Axon embedded controllers for real-time execution. For joint position control we rely on the elastomer quadrature encoders for feedback as explained before. However, we turn off the inner force feedback controllers described in Section \ref{sub:joint_position_control} to increase the joint position accuracy. This is important to land the small feet near the desired foot locations with minimal errors. 

The behavior resulting from integrating the new liquid-cooled viscoelastic in the DRACO biped robot with the WBC and TVR control algorithms is shown in Fig.~\ref{fig:dynamic_balancing}. DRACO is able to achieve unsupported dynamic balancing without falling. The accompanying video demonstrates this capability. The data is plotted in Fig.~\ref{fig:dynamic_balancing} for two consecutive steps. In the phase space plots of that figure, we can see the velocity of the torso being effectively reversed based on the TVR planner which aims precisely at achieving such outcome. Reversing velocity allows the robot to converge to the coordinate origin while dynamically stepping. Readers are referred to Appendix in \cite{fast} for more details about the phase space plot of the linear inverted pendulum. In addition, the use of liquid cooling safely regulates core winding temperatures of the electric motors. In that same figure, we can see that when turning off the liquid cooling system, the knee motors increases temperature beyond 130 $\si{\degree}$ which can damage the motor windings. In contrast, when turning on liquid cooling, the motor temperature remains below 60 $\si{\degree}$ all the time during balancing which is considered a very safe temperature. Overall, without the cooling system, we could not achieve agile locomotion for our lightweight system due to overheating.

\begin{figure}
    \centering
    \includegraphics[width=1.0\linewidth]{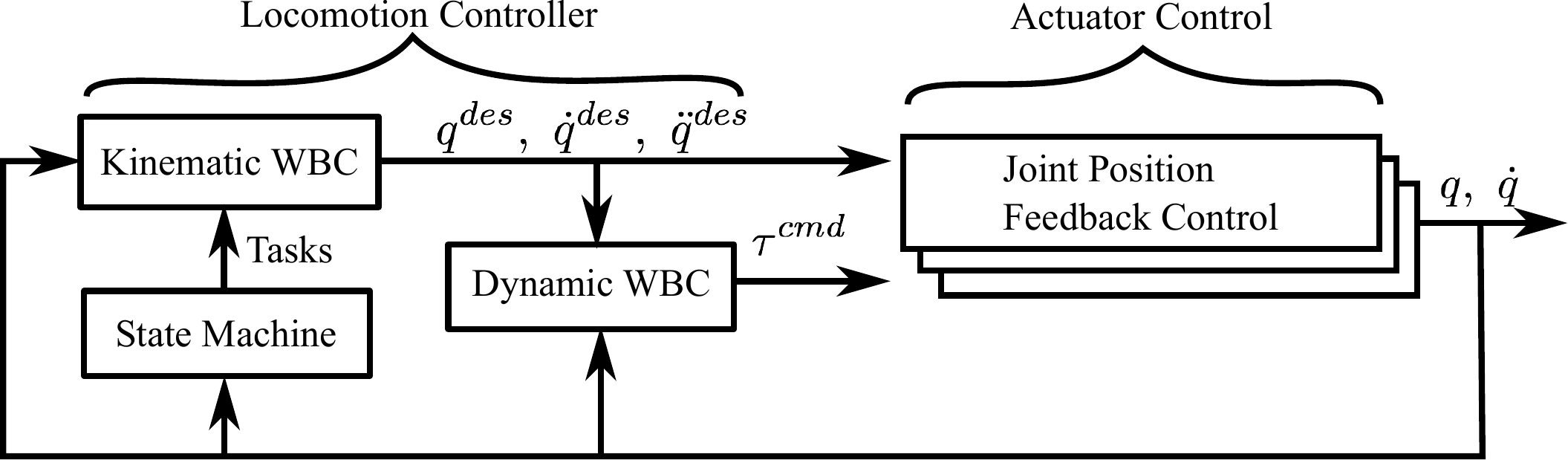}
    \caption{This figure describes a control architecture to achieve unsupported dynamic balancing.}
    \label{fig:walking_block}
    \vspace{-4mm}
\end{figure}

\section{Concluding remarks}
\label{sec:conclusion}
Overall, our main contribution has been on the control, and evaluation of a new adult-size humanoid bipedal robot, dubbed DRACO with control considerations on the VLCA actuators. DRACO is able to achieve unsupported dynamic balancing with only ten actuators, and despite the ankle actuators being much weaker than human ankles. This performance is possible due to a combination of mechanics that reduce distal mass, the use of high power dense actuators, high quality sensing, the integration of a robust planner, and stiff controllers that control the robot's body, foot, and joint positions with high accuracy.

In the future, we plan to remove the use of the motion capture system for CoM state estimation and rely on IMU and vision. For localization we plan to explore the integration of stereo RGB cameras for  dense range data and high frame-rates. Another part of this project will consists on the addition of an upper body with two robotic arms for loco-manipulation behaviors.

\section*{Acknowledgment}
The authors would like to thank the members of the Human Centered Robotics Laboratory at The University of Texas at Austin and the company Apptronik for their great help and support. This work was supported by the Office of Naval Research, ONR Grant \#N000141512507 and the National Science Foundation, NSF Grant \#1724360.

\addtolength{\textheight}{-12cm}   

\bibliographystyle{IEEEtran}
\bibliography{2019_RAL_DRACO.bib}

\end{document}